\documentclass{article}
\usepackage{ijcai19}

\usepackage{times}
\usepackage{soul}
\usepackage{url}
\usepackage[hidelinks]{hyperref}
\usepackage[utf8]{inputenc}
\usepackage[small]{caption}
\usepackage{graphicx}
\usepackage{amsmath}
\usepackage{booktabs}
\usepackage{algorithm}
\usepackage{algorithmic}
\usepackage{color}
\usepackage{subfig}

\newtheorem{problem}{\textbf{Problem}}
\newtheorem{proof}{\textbf{Proof}}
\newtheorem{theorem}{Theorem}
\newcommand{\probref}[1]{Problem~\ref{problem:#1}}
\newcommand{\algref}[1]{Algorithm~\ref{alg:#1}}
\newtheorem{definition}{\textbf{Definition}}

\newcommand{\secref}[1]{Section~\ref{sec:#1}}
\newcommand{\figref}[1]{Figure~\ref{fig:#1}}

\usepackage{multirow}

\title{Mindful Active Learning}

\author{
Zhila Esna Ashari\footnote{Contact Author}\And
Hassan Ghasemzadeh\\
\affiliations
School of Electrical Engineering \& Computer Science
\\Washington State University, Pullman WA, 99164, USA\\
\emails
\{z.esnaashariesfahan, hassan.ghasemzadeh\}@wsu.edu
}

\begin{document}

\maketitle

\begin{abstract}
We propose a novel active learning framework for activity recognition using wearable sensors. Our work is unique in that it takes physical and cognitive limitations of the oracle into account when selecting sensor data to be annotated by the oracle. Our approach is inspired by human-beings' limited capacity to respond to external stimulus such as responding to a prompt on their mobile devices. This capacity constraint is manifested not only in the number of queries that a person can respond to in a given time-frame but also in the lag between the time that a query is made and when it is responded to. We introduce the notion of {\it mindful active learning} and propose a computational framework, called {\it EMMA\footnote{Software for EMMA ({\it Entropy-Memory Maximization}) is available at https://github.com/zhesna/EMMA.}}, to maximize the active learning performance taking informativeness of sensor data, query budget, and human memory into account. We formulate this optimization problem, propose an approach to model memory retention, discuss complexity of the problem, and propose a greedy heuristic to solve the problem. We demonstrate the effectiveness of our approach on three publicly available datasets and by simulating oracles with various memory strengths. We show that the activity recognition accuracy ranges from $21$\% to $97$\% depending on memory strength, query budget, and difficulty of the machine learning task. Our results also indicate that EMMA achieves an accuracy level that is, on average, $13.5$\% higher than the case when only informativeness of the sensor data is considered for active learning. Additionally, we show that the performance of our approach is at most $20$\% less than experimental upper-bound and up to $80$\% higher than experimental lower-bound. We observe that mindful active learning is most beneficial when query budget is small and/or oracle's memory is weak, thus emphasizing contributions of our work in human-centered mobile health settings and for elderly with cognitive impairments.
\end{abstract}

\section{Introduction}
Coupled with machine learning algorithms, mobile and wearable devices are being increasingly utilized to monitor physical and mental health, and to improve human wellbeing through clinical interventions. To train accurate machine learning models, an adequate number of labeled sensor data is required. Unfortunately, models that are trained based on sensor data collected in controlled environments and laboratory settings perform extremely poorly when utilized in uncontrolled environments and outside clinics \cite{7945018}.

For the machine learning algorithms to be accurate, one needs to collect and label sensor data in end-user settings. Therefore, active learning is a natural choice for labeling the data where the end-user\footnote{The terms `end-user' and `oracle' are interchangeably used in this paper.} acts as the oracle agent and we iteratively query the user for correct labels \cite{Settles,Bachman}. In such a human-centered setting, it is critical to design active learning strategies that are mindful of the user's physical and cognitive capabilities. We recognize that human beings are limited in their capacity to respond, for example, to prompts on their mobile devices. This capacity constraint is usually manifested in the number of queries that a person can respond to in a given time-frame and in the difference between the time that a query is made and when it has been responded to. This issue is critical in wearable-based continuous health monitoring where the amount of sampled data is orders of magnitude more than what the end-user can possibly annotate \cite{redmond2014does,ma2019cyclepro}. In this paper, we introduce the notion of {\it mindful active learning} and propose {\it EMMA} to maximize the active learning performance taking informativeness of sensor data, query budget, and human memory into account. To the best of our knowledge, our work is the first study that combines informativeness of queried data with oracle's memory strength in a unified framework for active learning.

\subsection{Related Work}
Active learning has been widely used for human-in-loop learning tasks to iteratively interact with an expert user to retrieve necessary information for improved learning performance \cite{Huang,Sabato,Zuobing,Jun,Hoi,Shen,Sener}. Prior research studied the problem of active learning under constrained query budget for image and video recognition \cite{Vijayanarasimhan}. Active learning has been also studied in the context of activity recognition \cite{Stikic,Diethe}. For instance, accelerometer data annotated by end-users are used to detect physical activities \cite{Bao}. Active learning has been used for many other human-centered prediction problems beyond those that use wearable sensors and images \cite{Settles,Murukannaiah}. However, prior research does not take into account important cognitive attributes of human beings such as memory strength and forgetfulness of the events while designing active learning solutions. Current research makes an implicit assumption that either the oracle has a perfect memory that can precisely remember all the events, or each query is instantaneously responded to. None of these assumptions is realistic in continuous health monitoring settings where the end-user also acts as our oracle. In this paper, we attempt to take the first steps at integrating physical and cognitive attributes of the oracle with active learning. In particular, we account for (i) the oracle's physical capability to respond to active learning queries as measured by the number of queries that are made; and (ii) the oracle's cognitive capacity to respond to the queries as measured by memory retention, a function that combines memory strength of the oracle agent with the amount of delay in responding to the issued queries.

\subsection{Contributions}
Our contributions in this paper can be summarized as follows: (i) we introduce mindful active learning as a general approach for budget-aware and delay-tolerant active learning in human-in-the-loop health monitoring applications. Mindful active learning takes into account the possibility of forgetting the event of interest by the oracle based on the time difference between querying and the activity being performed constrained on a maximum number of queries; (ii) we propose an approach to model memory retention based on the Ebbinghaus forgetting curve \cite{Ebbinghaus}; (iii) we formulate mindful active learning as an optimization problem and show that the problem is NP-hard; (iv) we propose a greedy heuristic algorithm to solve the optimization problem; and (v) we evaluate the performance of EMMA on three datasets for activity recognition where the goal is to detect physical activities performed by the end-user.

\section{Problem Statement}
Assume that we are given a collection of sensor measurements recorded while wearable sensors are carried by the user during daily activities. Without loss of generality, we assume that the sensor measurements, referred to as sensor observations henceforth, are represented in a feature space. These observations need to be used to train a classifier for activity recognition. Because the sensor observations are unlabeled, we use active learning to construct a labeled training set by querying the end-user to annotate a subset of the observations. Because of the lag between querying a sensor observation and performing the activity associated with that observation, it is possible that the oracle is incapable of remembering the correct label. We assume that the likelihood of such a mislabeling is a function of the oracle's memory strength and the time lag.  Moreover, we assume that the number of allowed queries is constrained to a given value, referred to as {\it budget}, in order to minimize burden of the oracle. Therefore, we need to select a subset of sensor observations, upper-bounded by the given budget, such that the probability of obtaining correct and informative labels is maximized. This problem can be defined as follows.


\begin{problem}[Mindful Active Learning]
    Let $\mathcal{A}$ = \{$A_1$,$A_2$,..., $A_n$\} represent the set of activities that need to be recognized by the wearable system. We refer to this set as {\it activity vocabulary}. Furthermore, let $\mathcal{X}$ = \{$X_1$,$X_2$,..., $X_m$\} be a set of $m$ observations made by the sensors at times  $\mathcal{T}$ = \{$t_1$,$t_2$,..., $t_m$\}. The task of active learning is to query the user, with a predefined memory strength of $s$, at time $t_q$$\geq$$t_m$, to label sensor observations in $\mathcal{X}$ and train an activity recognition model using the labeled observations. The active learning process is constrained by limiting the number of queries to a given upper-bound budget, $B$.
\end{problem}

\subsection{Problem Formulation}
The task of selecting a subset of observations to be labeled by the oracle can be viewed as finding a set $\mathcal{Z}$ = \{$z_1$,$z_2$,..., $z_k$\} of $k$ observations in $\mathcal{X}$ (i.e., $\mathcal{Z}$ $\subseteq$ $\mathcal{X}$ and $k$$\leq$$B$) such that the misclassification error due to a model trained over $\mathcal{Z}$ is minimized. The goal is to select the best subset of unlabeled observations to form a candidate query set ($Z$) that results in an accurate classifier. In order to maximize the performance of the final classifier, we consider two criteria including informativeness of the candidate observation and the oracle's ability to remember the correct label at time $t_q$. We use $I_i$ and $M_i$ to refer to informativeness and memory measures for a given observation $X_i$ captured at time $t_i$.
  
This optimization problem can be formulated as 

\begin{alignat}{7}
          \label{eq1}
          \text{Maximize}        & \quad \sum_{i=1}^{m} a_i  \mathcal{E}(I_i, M_i)
\end{alignat}

Subject to:

\begin{alignat}{7}
          \label{eq2}
          &  \quad  X_i \in \mathcal{X} \\
          \label{eq3}
                      &  \sum_{i} a_i \leq B \\
          \label{eq4}
                      &  \quad a_i \in \{0,1\} 
\end{alignat}

\noindent where $X_i$ denotes the senor observation captured at time $t_i$, and $B$ represents the budget. The binary variable $a_i$ determines whether or not observation $X_i$ is selected for inclusion in $\mathcal{Z}$. The objective function in \eqref{eq1} attempts to maximize the total amount of expected gain ($\mathcal{E}$) given informativeness and memory for individual observations.

We propose to use uncertainty of the model with respect to a given observation as a measure of informativeness of that observation for inclusion in $\mathcal{Z}$. We also propose to measure memory by combining memory strength of the user with the expected error due to difference in time between query issuance and activity occurrence. The expected gain is then computed as

\begin{equation}\label{eq5}
\mathcal{E}(I_i, M_i)=\mathcal{E}(I_i)\mathcal{E}( M_i)
\end{equation}

To quantify informativeness ($I_i$) of observation $X_i$, we propose to use entropy, as shown in \eqref{eq6}, to measure how certain the model is about its predicted label for $X_i$.

\begin{equation}\label{eq6}
\mathcal{E}(I_i) = E_i = {-\sum_{j=1}^{n} P_{ij} \log P_{ij} }
\end{equation}

The term $E_i$ in \eqref{eq6} refers to entropy for observation $X_i$, and $P_{ij}$ represents the probability of $X_i$ being classified as $A_j$. Because the classifier is less certain to classify observations that carry a higher entropy, such observations will naturally be more informative if labeled and used for classifier retraining. Therefore, $E_i$ is a reasonable proxy for $\mathcal{E}(I_i)$. To quantify memory, $M_i$, for observation $X_i$, we define memory retention as follows. 

\begin{definition}[Memory Retention]
Memory retention, $R$, is defined as the probability of a human subject with a memory strength of $s$ being able to remember an event correctly after a given time, $t$, has elapsed.
\end{definition}

We use Ebbinghaus forgetting curve \cite{Ebbinghaus} to quantify memory retention. To this end, memory retention for observation $X_i$ is given by

\begin{equation}\label{eq7}
\mathcal{E}(M_i) = R_i = e^{-\Delta{t}_i/s} 
\end{equation}

\noindent where $\Delta{t}_i$ denotes the difference in time between occurrence of the event represented by observation $X_i$ and issuance of the query (i.e., $t_q$). Furthermore, the term $s$ represents memory strength, which is a user-specific value. The memory retention in \eqref{eq7} is naturally a measure for receiving correct labels from the oracle. Therefore, the memory retention $R_i$ can be used to quantify the memory for observation $X_i$ (i.e., $\mathcal{E}( M_i)=R_i$).

\begin{problem}[Entropy-Memory Maximization]\label{problem:emma}
By replacing $\mathcal{E}(I_i)$ and $\mathcal{E}(M_i)$ in \eqref{eq5} with entropy measure in \eqref{eq6} and memory retention in \eqref{eq7}, we can rewrite the objective function in \eqref{eq1} as follows.

\begin{alignat}{7}
\label{eq8}
\text{Maximize}        & \quad \sum_{i=1}^{m} \sum_{j=1}^{n} a_i e^{-\Delta{t}_i/s} (-P_{ij} \log P_{ij})
\end{alignat}

We note that this new objective function in \eqref{eq8} is subject to the same constraint in \eqref{eq2}--\eqref{eq4}. We refer to this formulation of mindful active learning as Entropy-Memory Maximization (EMMA).
\end{problem}

\subsection{Problem Complexity}\label{sec:complexity}
In this section, we discuss complexity of the entropy-memory maximization. We first show that the problem is NP-hard.

\begin{theorem}
	The entropy-memory maximization in \probref{emma} is NP-Hard.
\end{theorem}
\begin{proof}
Proof by reduction from the well-known Knapsack problem \cite{pisinger2005hard}. Details are eliminated for brevity.
\end{proof}

To provide some insight into the input size of the EMMA problem, consider a naive approach (i.e., brute-force), in which we cycle through all subsets of the unlabeled dataset with the number of elements less than or equal to $B$. For each subset, we can generate all possible orderings of the sensor observations within that subset. For each ordering, we draw one observation from the subset at a time, compute entropy and memory retention values, query the oracle, and train a classifier with the observations labeled so far. By repeating this process for all subsets and all orderings of the observations within each subset, we can finally choose the ordering and subset that achieve the maximum value for the objective function in \eqref{eq8}. Although this brute-force approach is impractical for real-world deployment, an analysis of the complexity of such an approach provides insight into the relationship between the complexity function and various problem parameters.

\begin{theorem}
	The time complexity of the brute-force approach for EMMA is exponential in the query budget, $B$.
\end{theorem}
\begin{proof}

It is straightforward to see that the overall time complexity of the brute-force solution is $O(|S|)$ where $|S|$ refers to the total number of orderings of all subsets of unlabeled set $\mathcal{X}$ with a subset size less than or equal to $B$. For each subset of size $b$, there exist $b!$ orderings of the observations that reside within the subset. Therefore,  $|S|$ is given by

\begin{equation}\label{eq12}
|S| = \sum_{b=1}^{B}{m\choose b } \times b!= \sum_{b=1}^{B}{\frac{m!}{(m-b)!}} 
\end{equation}

\noindent where $m$ denotes the size of the unlabeled set $\mathcal{X}$. The equation in \eqref{eq12} can be presented as

\begin{equation}\label{eq13}
|S| = \sum_{b=1}^{B} \prod_{k=1}^{b} (m-b+k)
\end{equation}

Because in real-world scenarios, the number of recorded sensor observations is orders of magnitude higher than the number of queries that the oracle can respond to (i.e., $B \ll m$), we can write

\begin{equation}\label{eq14}
\prod_{k=1}^{b} (m-b+k) \approx m^b
\end{equation}

Therefore, \eqref{eq13} can be rewritten as

\begin{equation}\label{eq15}
|S| \approx m+m^2+...+m^B \approx {\dfrac{m^{(B+1)}-1}{m-1}} \approx m^B
\end{equation}

Therefore, the time complexity of the naive approach is $O(m^B)$.
\end{proof}

\section{Problem Solution}\label{sec:solution}
Recall that the goal of the optimization problem in EMMA is to select a subset of observations of high-entropy such that their labels are likely of being remembered by the oracle. Because the problem is NP-hard in nature, we design a greedy heuristic algorithm, shown in Algorithm \ref{alg:1}, to solve the problem.

\begin{algorithm}[tbh!]
\caption{Greedy algorithm for EMMA.}
\label{alg:1}
\begin{algorithmic}
\STATE \textbf{Input:}  $\mathcal{X}$ (unlabeled observations), $B$ (budget)
\STATE \textbf{Output:} $\mathcal{Z}$ (labeled observations)
\FOR{$b$ = $1$ to $B$}
\STATE Compute $\mathcal{E}(I_i, M_i)$ for all $X_i \in \mathcal{X}$ using \eqref{eq5}--\eqref{eq7}
\STATE Find $X_i \in \mathcal{X}$ with highest value of $\mathcal{E}(I_i, M_i)$
\STATE Remove $X_i$ from $\mathcal{X}$
\STATE Query oracle to annotate $X_i$, and add labeled $X_i$ to $\mathcal{Z}$
\STATE Retrain model $\mathcal{M}$ using labeled items in $\mathcal{Z}$
\ENDFOR
\end{algorithmic}
\end{algorithm}

This simple greedy algorithm iteratively chooses the best candidate observation from the set of unlabeled observations $\mathcal{X}$ with the highest expected gain, shown in \eqref{eq5}, assuming that the informativeness and memory are measured as given in \eqref{eq7} and \eqref{eq8}. Note that after moving an observation from $\mathcal{X}$ to $\mathcal{Z}$, model $\mathcal{M}$ is retrained using the labeled data in $\mathcal{Z}$. This procedure is repeated until the entire budget is consumed.

\begin{theorem}
	The greedy approach shown in \algref{1} provides a $\frac{1}{2}$-approximation for EMMA.
\end{theorem}
\begin{proof}
	Proof is eliminated for brevity.
\end{proof}

\begin{theorem}
	The time complexity of \algref{1} is linear in $m$, the number of original unlabeled observations in $\mathcal{X}$.
\end{theorem}
\begin{proof}
	The `for' loop iterates for `B' times. During each iteration, we need to (i) compute $\mathcal{E}(I_i, M_i)$ for the remaining elements in $\mathcal{X}$; and (ii) find $X_i \in \mathcal{X}$ with highest value of $\mathcal{E}(I_i, M_i)$. Both of these operations require $O(|\mathcal{X}|)$ to complete assuming a constant time complexity for computing $\mathcal{E}(I_i, M_i)$ for a given $X_i$. Therefore, the time complexity of \algref{1} is $B \times O(|\mathcal{X}|)$. Because $|\mathcal{X}|$ is initially $m$ and decreases by one in each iteration, the total number of times to compute $\mathcal{E}(I_i, M_i)$ is given by
	
	\begin{equation}\label{eq18}
	\sum_{b=1}^{B} (m-b-1) = O(m.B)
	\end{equation}
	Therefore, the time complexity of \algref{1} is $O(m.B)$
\end{proof}

\section{Validation}
This section presents our validation approach, results on the performance of our mindful active learning algorithms, and a comparative analysis. Throughout this section, we use {\it EMMA} to refer to the greedy algorithm presented in \secref{solution}. We note that the naive approach discussed in \secref{complexity} has an exponential time complexity, which makes it impractical given the large amounts of data collected during continuous health monitoring using wearable sensors. 

\subsection{Experimental Setup}\label{sec:setup}
To assess the performance of EMMA, we used three real-world, publicly-available, sensor-based datasets including  HART \cite{Anguita}, DAS \cite{Altun1,Barshan,Altun2}, and AReM \cite{Palumbo}. HART contained sensor data collected with $30$ human subjects during six activities. For DAS, data from eight human subjects and $10$ activity classes were used. The AReM dataset involved six activities and one human subject.

We also simulated oracles with various memory strengths. The goal was to measure the likelihood of an incorrect label provided by the oracle given the memory strength and the time that has passed from capturing the queried observation. To simulate the oracle's remembering of the event associated with a queried sensor observation $X_i$ based on a given memory strength value $s$, we first computed memory retention $R_i$ for $X_i$ using \eqref{eq7}. We then assigned the correct label with the probability $R$ and incorrect label with the probability $(1-R)$. To alleviate the effect of randomness in our simulation, we repeated each experiment $30$ times and report the average results.

For comparative analysis, we assessed the performance of EMMA against the following active learning approaches:

\begin{itemize}
	\item {\bf EMA} (Entropy Maximization): this approach solves the optimization problem based on informativeness of the queried observations only, while still considering oracle's memory deficiency effects in responding to the queries. In this case, $\mathcal{E}(I_i)$=$E_i$ and $\mathcal{E}(M_i)=1$. 
	
	\item {\bf MMA} (Memory Maximization): this algorithm aims to maximize the objective function for memory retention only. This means, the algorithm tends to choose observations whose probability of remembering the label is higher than others. In this case, $\mathcal{E}(I_i)$=$1$ and $\mathcal{E}(M_i)$=$R_i$. 
	
	\item {\bf Upper-bound (UB)}: this refers to the case where no erroneous labels exist as a result of memory deficiency. That is, we assume that the oracle's memory is perfect, as a result of which the optimization problem aims to maximize for entropy only.
	
	\item {\bf Lower-bound (LB)}: this refers to the case where the oracle's memory is low and the observations are chosen randomly. Therefore, the informativeness of the queried observation is not considered as a parameter.
\end{itemize}
 
Throughout our experiments, we used SVM (Support Vector Machine) as our underlying activity recognition classifier. This choice was made based on extensive experiments on the datasets, where we considered four types of machine learning algorithms including SVM with linear kernel, SVM with RBF (Radial Basis Function) kernel, logistic regression, and decision tree. These algorithms achieved $92$\%, $83$\%, $90$\%, $88$\% average accuracy, respectively, over all subjects. Note that the presented methodologies are independent of choice of the activity recognition classifier. We also assumed that a small number of labeled samples are initially available in $\mathcal{Z}$. In our experiments, $\mathcal{Z}$  initially had two randomly selected labeled samples. This allowed us to learn an initial model to calculate primary entropy values for active learning.

\subsection{Performance of EMMA}
As a first analysis, we evaluated the performance of EMMA and examined how its accuracy changes as various parameters of the algorithm change. To this end, we conducted multiple experiments by changing algorithm parameters including query budget, $B$, and memory strength, $s$. After EMMA constructed a labeled dataset, we trained an activity recognition classifier using the labeled data.

\figref{fig2} shows the accuracy of the learned classifier on the three datasets as the query budget ranges from $5$ to $200$. The graphs show the accuracy values for five different memory retention levels including $R_1$ ($10$\%--$99$\%), $R_2$ ($20$\%--$99$\%), $R_3$ ($30$\%--$99$\%), $R_4$ ($50$\%--$99$\%), and $R_5$ ($70$\%--$99$\%). Recall that the memory retention $R$ refers to the probability of remembering label for a given sensor observation in the dataset. However, memory retention is also a function of both memory strength and time different between query issuance and labeling (i.e., lag or delay). We also note that sensor observations have different time stamps and therefore represent different time delays. Therefore, each memory retention value (i.e., $R_1$ to $R_5$) in \figref{fig2} us represented as a range rather than a fixed value. Note that, as the time window size to capture each sensor observation from raw sensor readings varies across the datasets, we used different memory strength values to obtain the same memory retention intervals/ranges to provide comparable results.

\begin{figure}[t]
	\subfloat[HART \label{fig:HART}]{%
		\includegraphics[ width=0.33\linewidth]{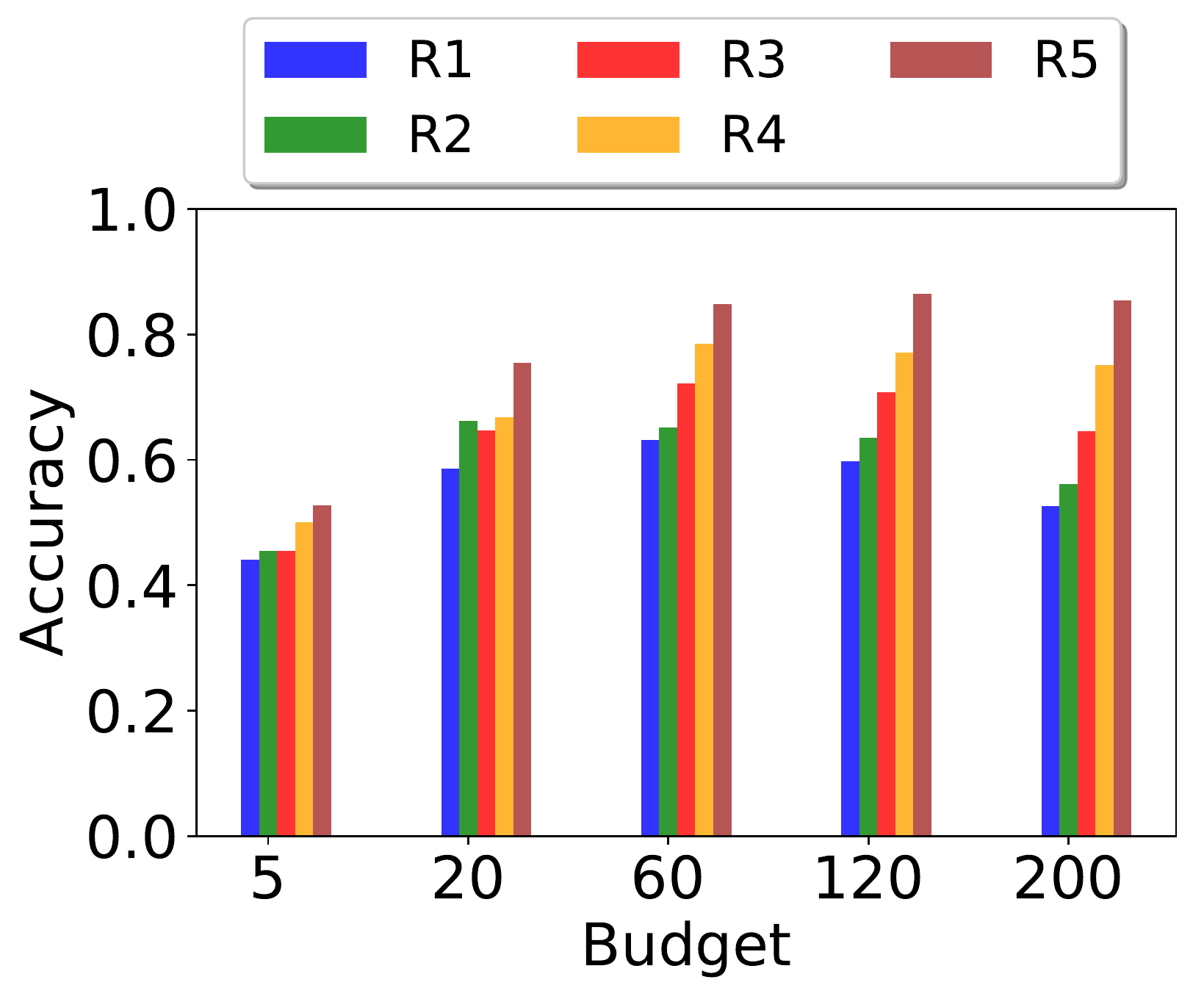}}
	\hspace*{-0.3em}
	\subfloat[DAS \label{fig:DAS} ]{%
		\includegraphics[ width=0.33\linewidth]{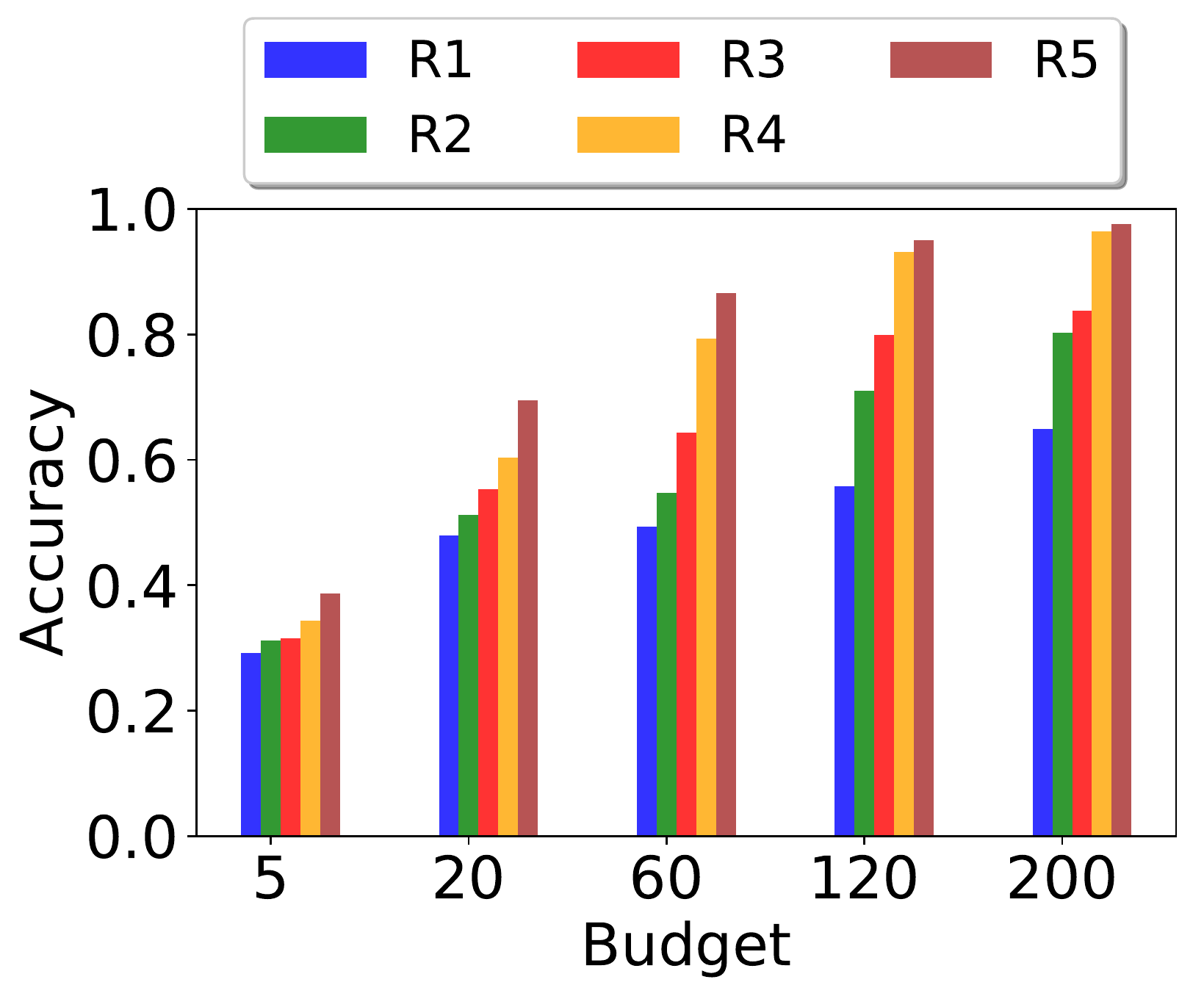}}
	\hspace*{-0.3em}
	\subfloat[AReM \label{fig:AReM}]{%
		\includegraphics[ width=0.33\linewidth]{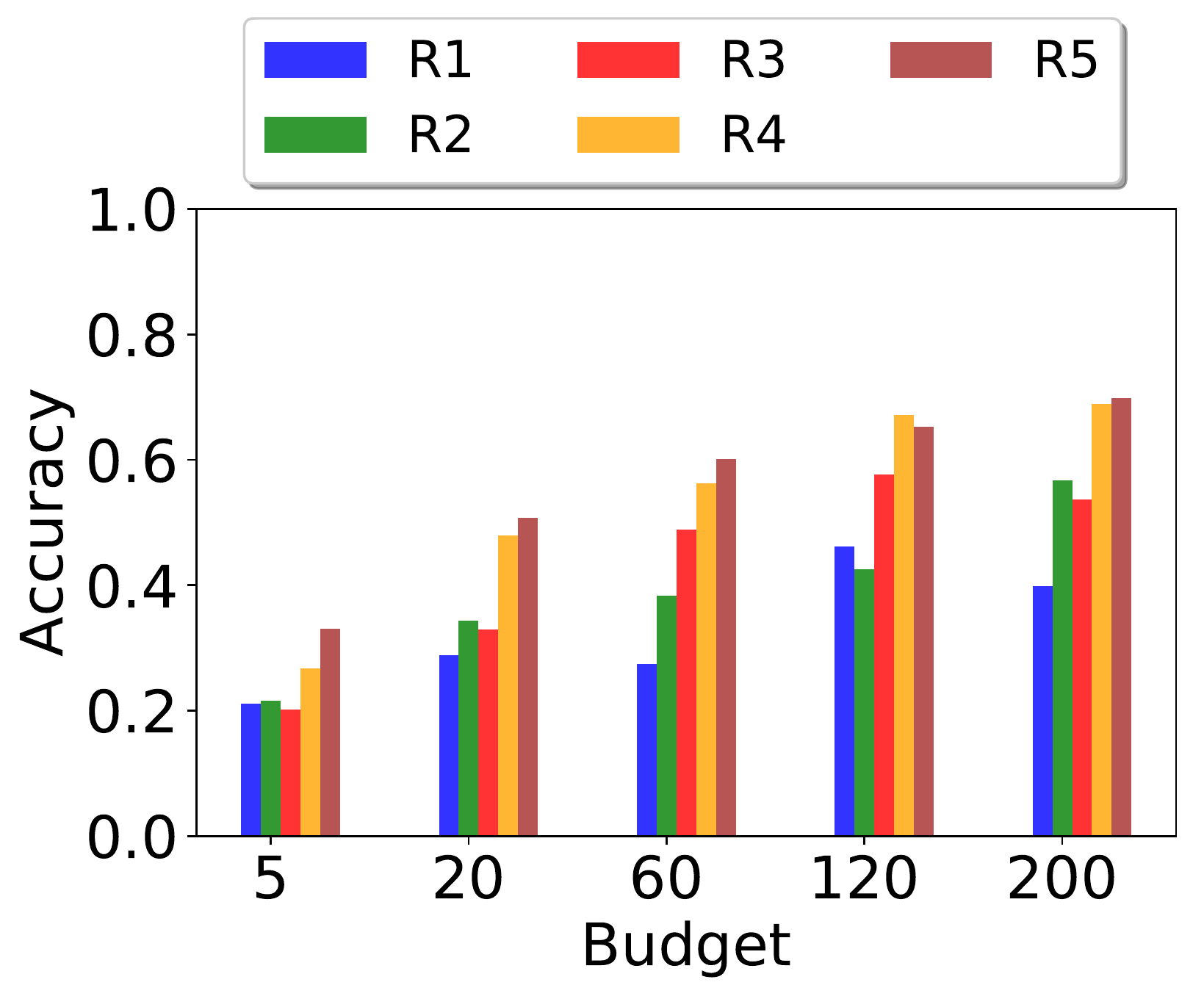}}
	\vspace{-2mm}
	\caption{\label{fig:fig2} Accuracy of activity recognition model learned using EMMA on three datasets (HART, DAS, and AReM) as a function of query budget for five different memory retention levels ($R_1$--$R_5$).}
	\vspace{-4mm}
\end{figure}

As \figref{fig2} suggests, the activity recognition accuracy improves for each memory retention value as the query allowance (i.e., budget) increases. Also, as the memory strength grows, the accuracy improves for each budget value. The minimum accuracy, achieved with least budget and weakest memory is $44\%$, $29\%$ and $21\%$ for HART, DAS and AReM datasets, respectively. The accuracy reaches its maximum values of $85.3\%$, $97.5\%$ and $70\%$ with greatest budget and strongest memory for the three datasets. An interesting observation from \figref{fig2} is the performance decline in weak memory cases. It is generally known that increasing the number of labeled observations using active learning improves the classifier performance. However, we can see that adding a large number of labeled observations may reduce the accuracy if the memory retention is low due to a weak memory. An example of such cases is shown in \figref{HART} for retention levels $R_1$, $R_2$, and $R_3$. In all these cases, the activity recognition performance improves up to a certain point (e.g., $B$=$60$) and it starts dropping after that point, mainly because more incorrect labels are added to the dataset resulting in a less accurate model. For instance, the accuracy for $R_1$, as the weakest memory, starts from $44\%$ for $B$=$5$,  reaches up to $63\%$ for $B$=$60$, and again drops to $52.5\%$ for $B$=$200$. Another interesting observation is that the accuracy saturates after certain point with strong memories. This can be observed, for example, in \figref{HART} for memory retention levels $R_4$ and $R_5$ where the accuracy plateaus after acquiring $60$ labeled observations and stays around $77\%$ and $85\%$ for higher budget values. The reason is that, in this dataset, each added observation is very informative and the classifier learns faster and needs fewer observations to achieve its maximum capability. 

The difference in accuracy among different datasets in \figref{fig2} can be explained by the quality of the data and the difficulty of the recognition task. For example, the DAS dataset (i.e., \figref{DAS}) contains highly discriminative features over different activities. As a result, adding a few number of correctly labeled observations has a significant impact on improving the classifier performance. For instances, when considering $R_5$, the change in accuracy when $B$ grows form $5$ to $20$ is $23\%$ for HART and $30\%$ for DAS. On the contrary, AReM dataset (i.e., \figref{AReM}) contains less discriminative features and the window size over which the features are calculated is small, leading to needing a lot more labeled observations to learn new activities. Therefore, the growth in the performance of the classifier in \figref{AReM} is slow compared to other datasets. Specifically, the amount of improvement in accuracy is $17\%$ for $R_5$ when $B$ increases form $5$ to $20$ in AReM dataset, while this improvement is $23\%$ and $30\%$ for HART and DAS datasets, respectively.

\begin{figure}[t!]
	\subfloat[$R_1$ \label{$R_1$}]{%
		\includegraphics[ width=0.33\linewidth]{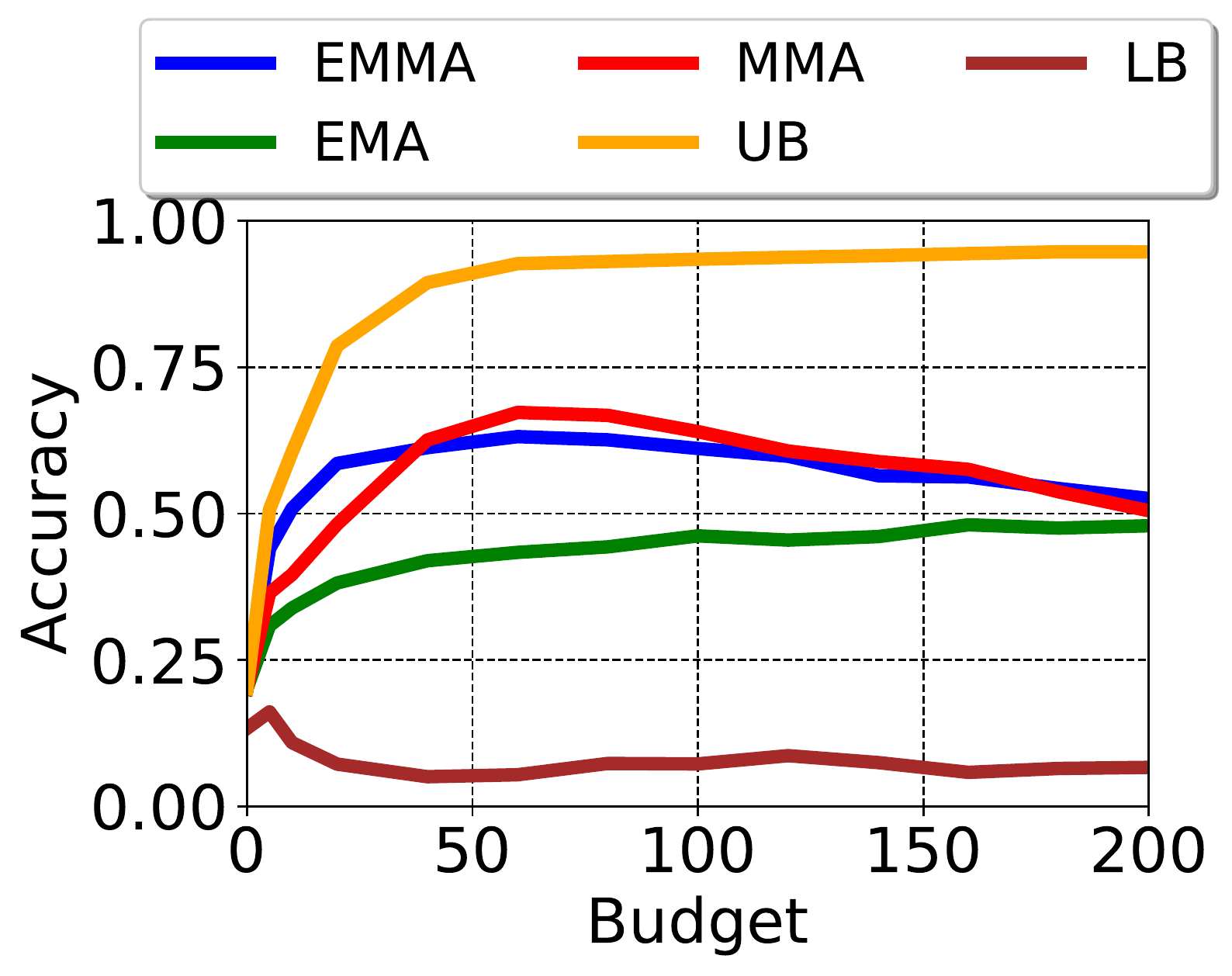}}
	\hspace*{-0.3em}
	\subfloat[$R_2$ \label{$R_2$} ]{%
		\includegraphics[ width=0.33\linewidth]{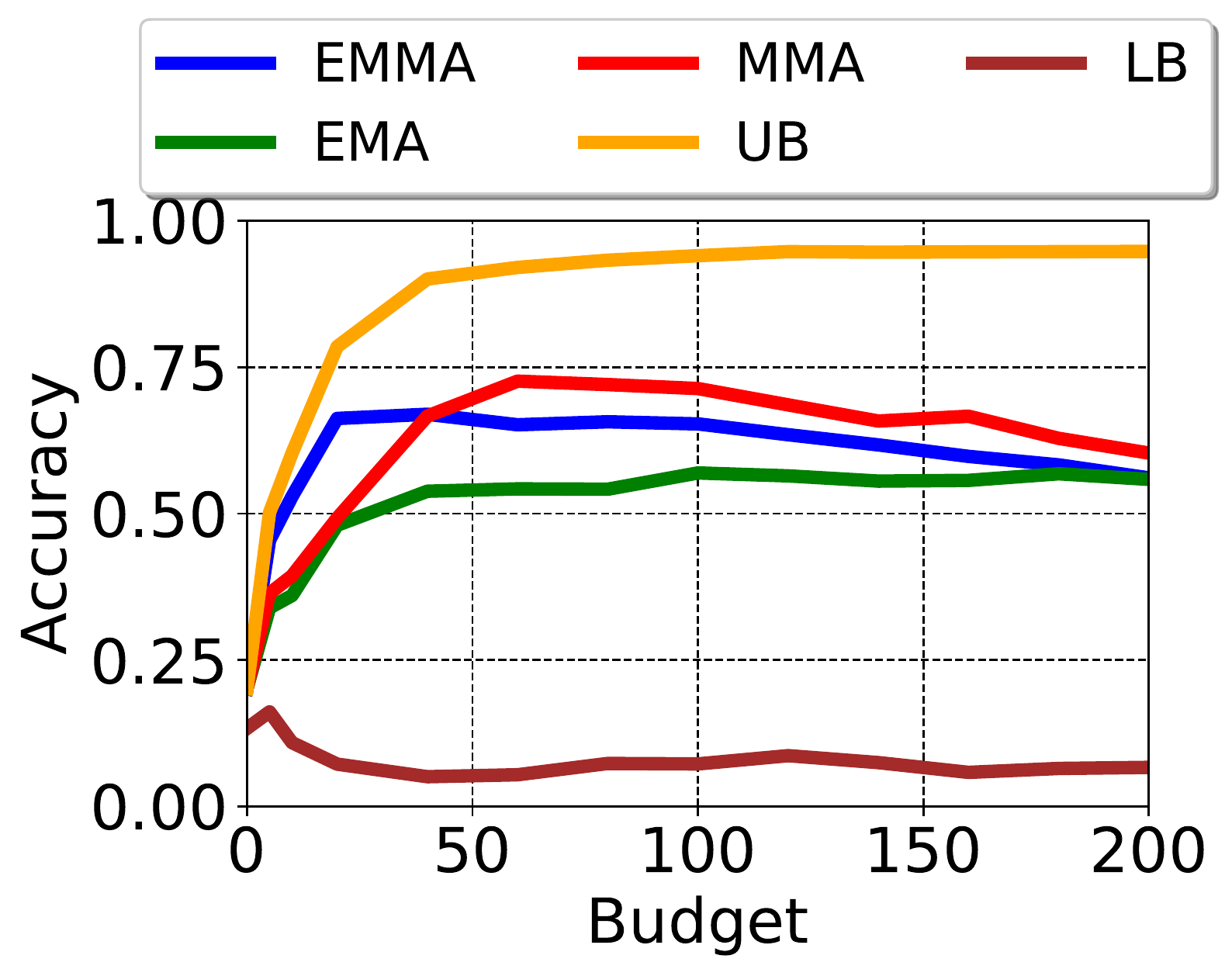}}
	\hspace*{-0.3em}
	\subfloat[$R_3$ \label{$R_3$}]{%
		\includegraphics[ width=0.33\linewidth]{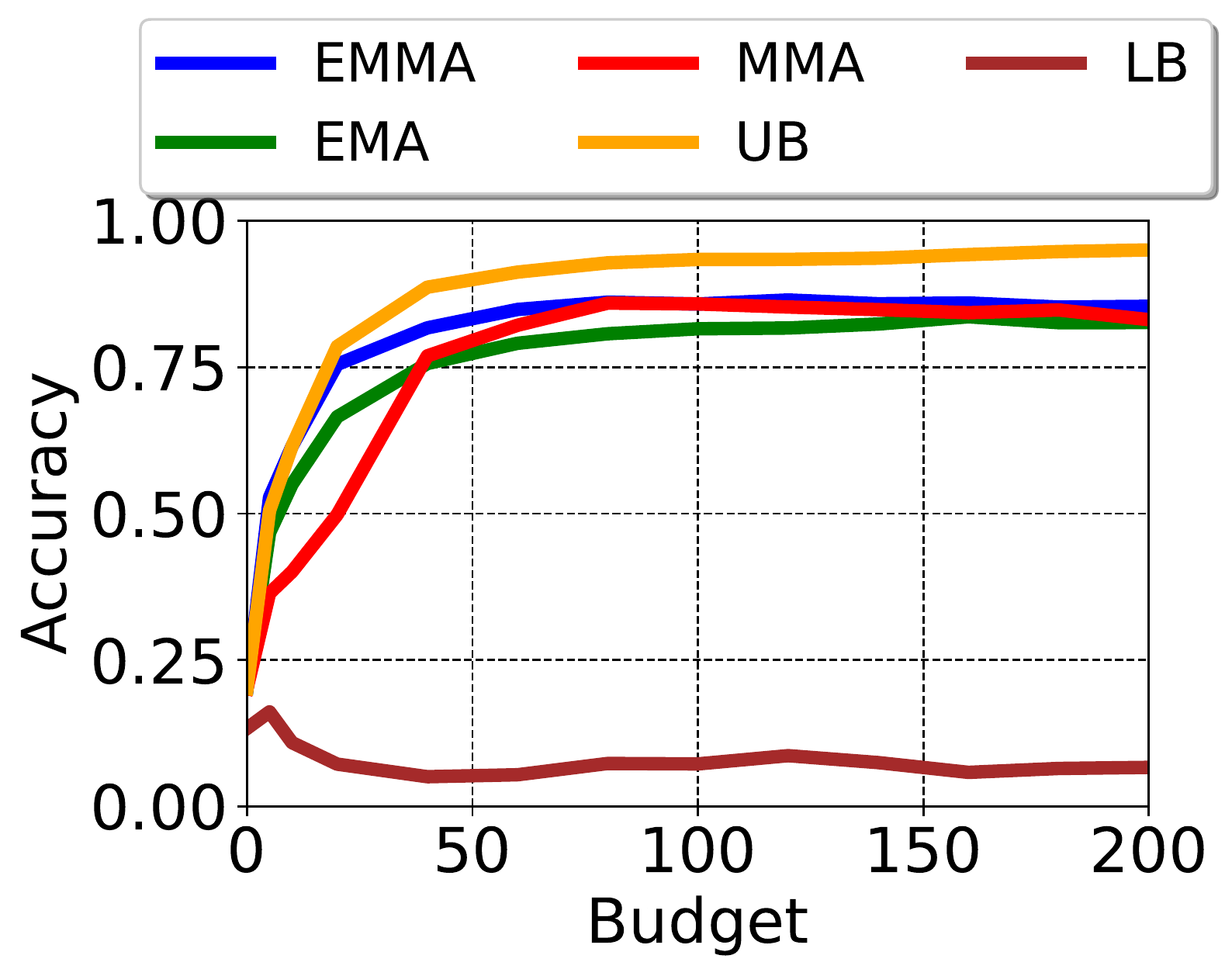}}
	\vspace{-2mm}
	\caption{\label{fig:fig3} Activity recognition accuracy of various active learning algorithms on HART dataset using three different memory retention levels, $R_1$ (low), $R_2$ (medium), and $R_3$ (high).}
	\vspace{-4mm}
\end{figure}

\subsection{Comparative Analysis}
We compared the performance of EMMA with that of several active learning approaches including EMA, MMA, upper-bound (UB), and lower-bound (LB). See \secref{setup} for a description of these alternative approaches. For brevity, we focus on accuracy as our main performance measure here. Similar to our analysis in the previous section, we examined the performance of each algorithm while the budget value ranged from $0$ to $200$ and the memory strength was set such that different memory retention levels, $R_1$ ($10$\%--$99$\%), $R_2$ ($25$\%--$99$\%), and $R_3$ ($70$\%--$99$\%) were obtained. The results of this comparative analysis are shown in \figref{fig3}, \figref{fig4}, and \figref{fig5} for HART, DAS, and AReM datasets, respectively. The accuracy numbers in these figures represent average values computed over all users in each dataset.

\begin{figure}[t!]
   \subfloat[$R_1$ \label{$R_1$}]{%
      \includegraphics[ width=0.33\linewidth]{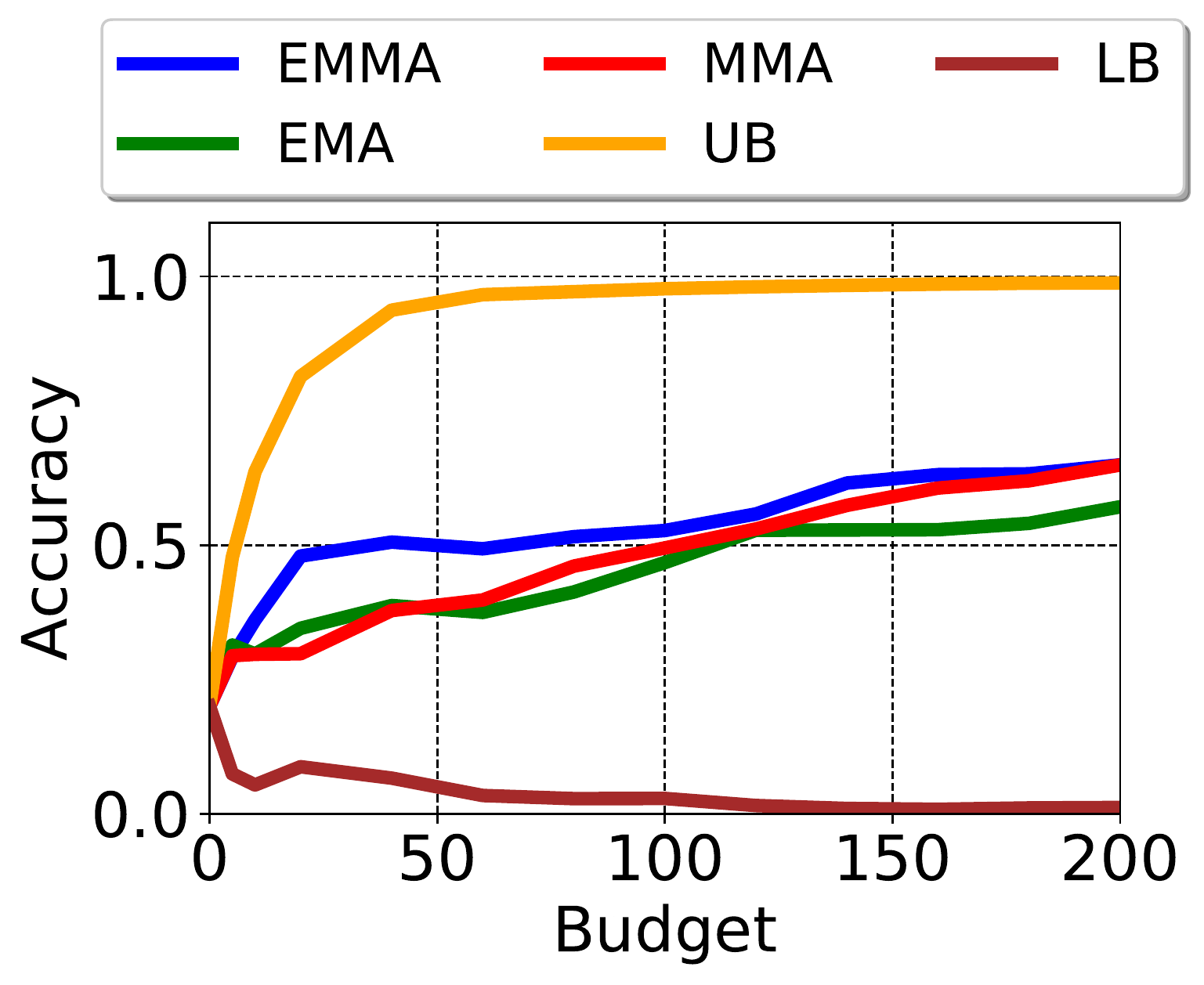}}
\hspace*{-0.3em}
   \subfloat[$R_2$ \label{$R_2$} ]{%
      \includegraphics[ width=0.33\linewidth]{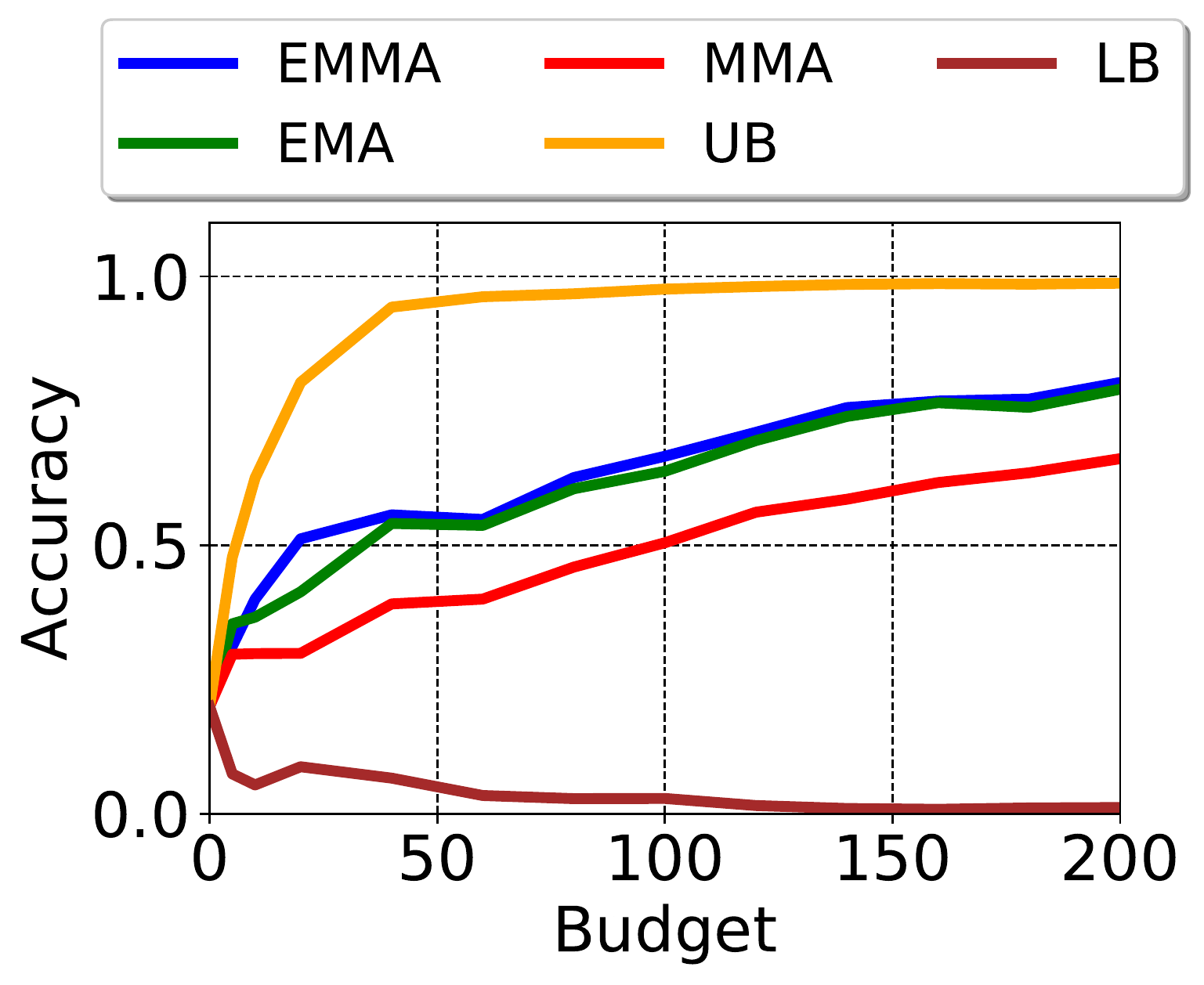}}
\hspace*{-0.3em}
   \subfloat[$R_3$ \label{$R_3$}]{%
      \includegraphics[ width=0.33\linewidth]{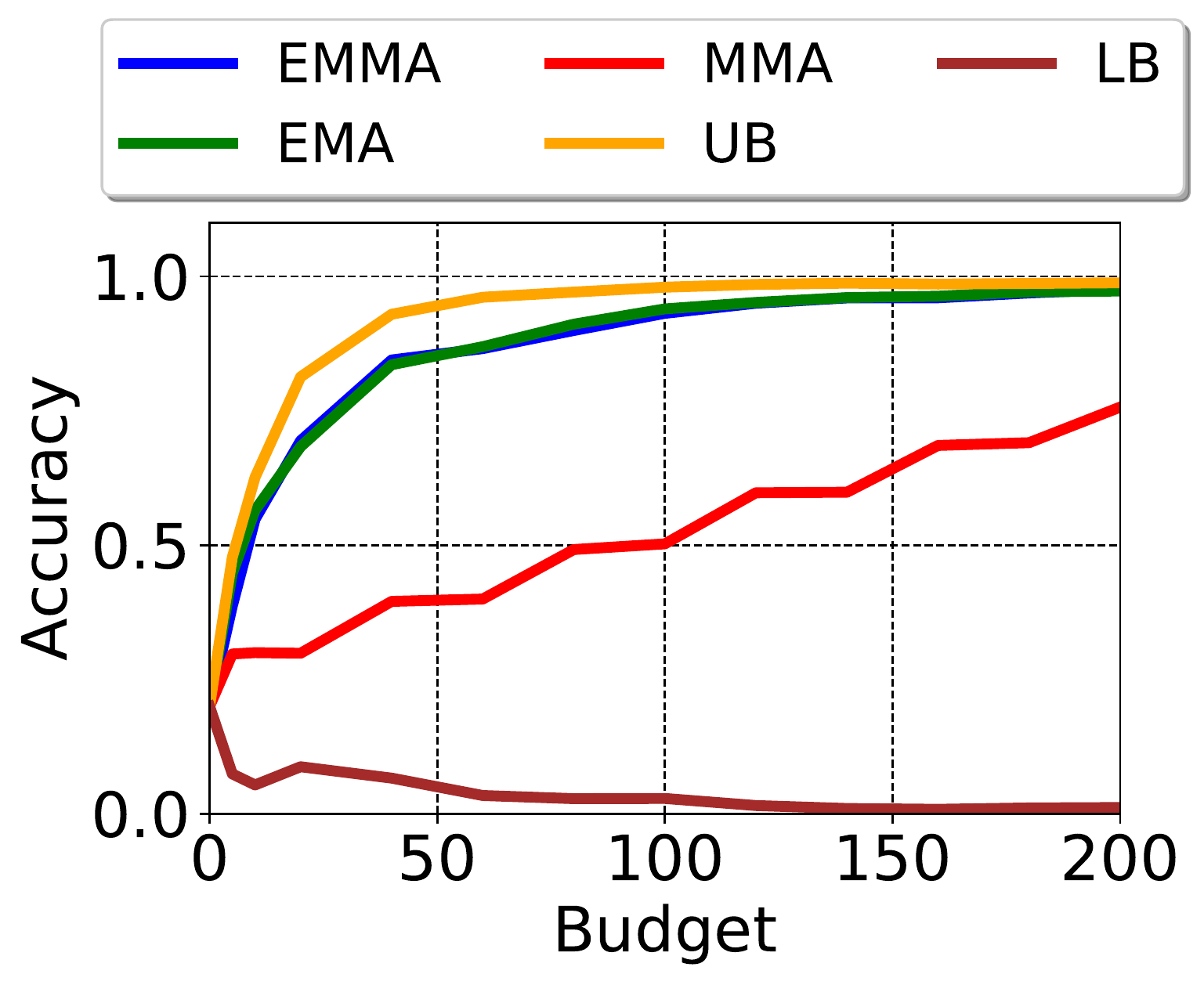}}
     \vspace{-2mm}
\caption{\label{fig:fig4} Activity recognition accuracy of various active learning algorithms on DAS dataset using three different memory retention levels, $R_1$, $R_2$, and $R_3$. }
\vspace{-4mm}
\end{figure}

The results show that the performance improvement due to using EMMA over EMA and MMA is most notable when the amount of budget is small and the memory is weak (i.e., memory retention level of $R_1$). This observation emphasizes the importance of considering both uncertainty and memory retention in health monitoring applications where data collection is extremely costly and end-users may be cognitively impaired due to aging. It can be seen from the results that EMMA achieves an average accuracy that is $13.5$\% higher than that of EMA and $14$\% higher than that of MMA for cases of weaker memory and smaller budget. Moreover, we note that, for all datasets, as the memory becomes stronger, EMMA and EMA methods converge and achieve accuracy values closer to the upper-bound. Additionally, the performance of EMMA is at most $20$\% less than the experimental upper-bound and up to $80$\% higher than the experimental lower-bound, on average. Furthermore, for HART dataset, MMA converges with EMMA and EMA as well, while for other datasets, the performance of MMA improves very slower. The reason is that MMA tends to select the observations sequentially in time and ignores the informativeness of the observations. Therefore, if the oracle repeats an activity for a long time (e.g., `sleeping', `watching TV'), this method continues to querying for the same activity, resulting in a model that is incapable of recognizing other activities. Consequently, the active learning process needs to query for many new observations in order to learn new activities (see results for DAS and AReM datasets). Note, as mentioned previously, AReM contains less discriminative features compared to the two other datasets. This is the reason why the performance of MMA is closer to the lower-bound on AReM.

A notable observation is difference in the performances of MMA and EMMA on HART and AReM datasets as shown in \figref{fig3} and \figref{fig5}. Specifically, MMA outperforms EMMA slightly in larger budgets in \figref{fig3}, but it performs much worse than EMMA in \figref{fig5}. This is because of switching fast between unknown activities in the HART dataset (\figref{fig3}), which results in MMA learning new activities at a faster pace compared to how MMA learns in the AReM dataset (\figref{fig5}). This observation suggests that the performance of MMA is highly dependent on the dataset (i.e., sequence of activities that occur). 

We can also compare the performance of EMMA with that of EMA in these two datasets. As shown in \figref{fig3}, EMMA outperforms EMA on HART. However, as shown in \figref{fig5}, EMMA and EMA achieve a comparable performance on AReM. This can be explained as follows. Because activities in the HART dataset change frequently over time, considering time delays during optimization (i.e., in EMMA) contributes to an improved performance. In contrast, the performance of EMA is highly dependent of the nature of the occurring events.

 

\begin{figure}[t!]
	\subfloat[$R_1$ \label{$R_1$}]{%
		\includegraphics[ width=0.33\linewidth]{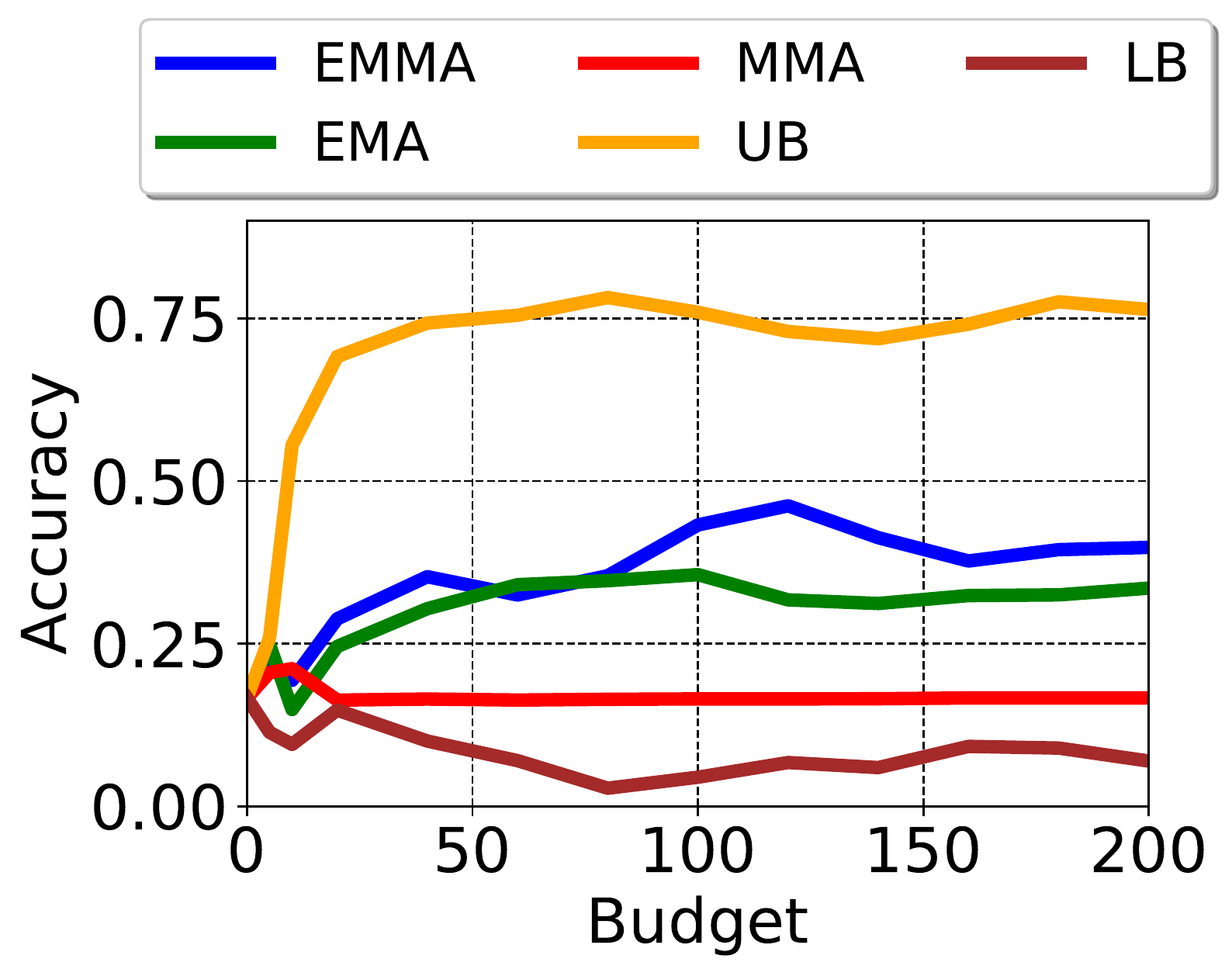}}
	\hspace*{-0.3em}
	\subfloat[$R_2$ \label{$R_2$} ]{%
		\includegraphics[ width=0.33\linewidth]{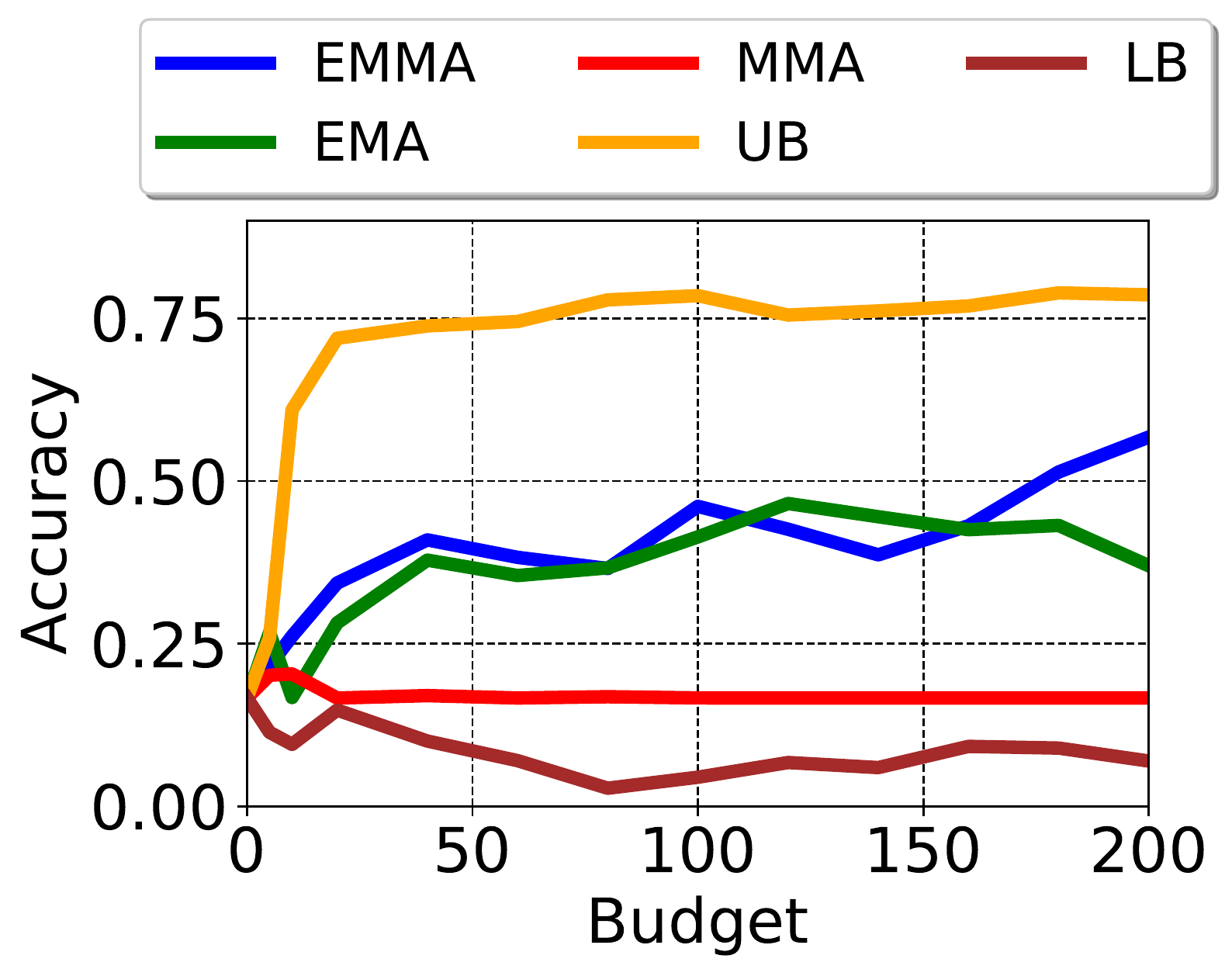}}
	\hspace*{-0.3em}
	\subfloat[$R_3$ \label{$R_3$}]{%
		\includegraphics[ width=0.33\linewidth]{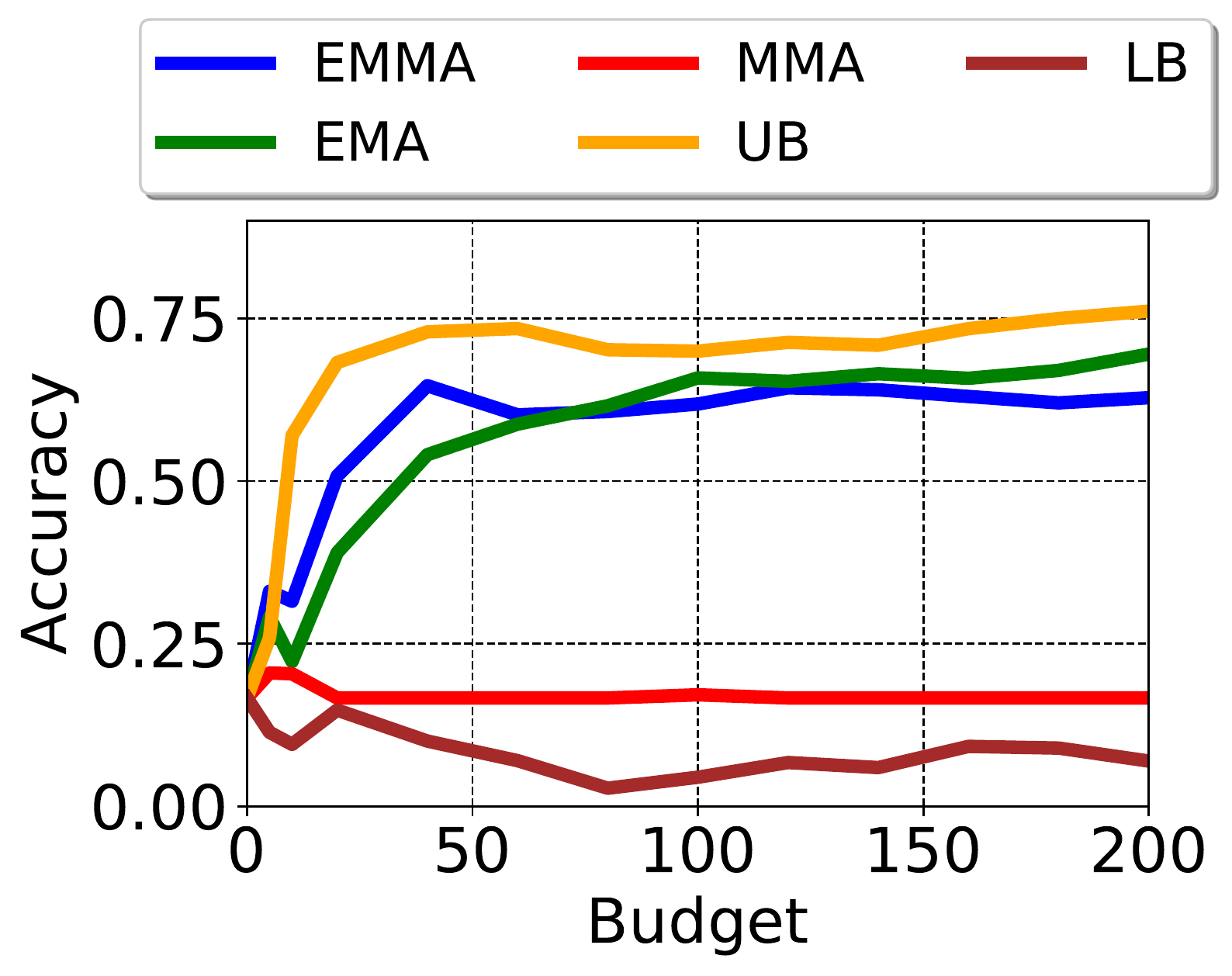}}
	\vspace{-2mm}    
	\caption{\label{fig:fig5} Activity recognition accuracy of various active learning algorithms on AReM dataset using three different memory retention levels, $R_1$, $R_2$, and $R_3$. }
	\vspace{-4mm}
\end{figure}

\section{Conclusions and Future Work}
Prior research on active learning takes informativeness of data and query budget into account when selecting the data for annotation. In this paper, we showed that cognitive constrains of the oracle are of significant importance that can greatly compromise the active learning performance. We posed an optimization problem to combine data uncertainty with memory retention for its use in ubiquitous and mobile computing applications. We showed that this problem is NP-hard and derived a greedy approximation algorithm to solve the proposed mindful active learning problem. Our extensive analyses on three publicly available datasets showed that EMMA achieves up to $97$\% accuracy for activity recognition using wearable sensors. We also showed that integrating memory retention improves the active learning performance by $16$\%.

Our results indicate that the performance of EMMA improves when oracle's memory is stronger or the query budget is higher. However, we noticed that increasing the budget does not improve the accuracy when highly accurate labels are available due to strong memory retention. We also noted that the active learning performance can decrease with increased budget if the oracle's memory is weak. Finally, the gap between the performance of EMMA and other algorithms is most notable with small budgets and weak memories. 

In this paper, we focused on pool-based active learning. Our ongoing work involves developing mindful active learning strategies that make query decisions on-the-fly as wearable sensor data become available in real-time.

Our future work will also focus on conducting user studies that involve patients with cognitive impairments where we can assess the oracle's memory retention quantitatively. Moreover, in this paper, we assumed that the oracle responds to an issued query instantaneously. We plan to investigate active learning solutions that take into account the possibility of delayed responses through context-sensitive active learning.


\section*{Acknowledgments}
This work was supported in part by the United States National Science Foundation under grant CNS-1750679. Any opinions, findings, conclusions, or recommendations expressed in this material are those of the authors and do not necessarily reflect the views of the funding organizations.

\newpage
\bibliographystyle{named}
\bibliography{myrefs}

\end{document}